\documentclass{article}

\usepackage{arxiv}

\usepackage{hyperref}       % hyperlinks
\usepackage{url}            % simple URL typesetting
\usepackage{amsfonts}       % blackboard math symbols
\usepackage[round]{natbib}
\usepackage{graphicx}
\usepackage{xcolor}
\usepackage[ruled,linesnumbered,boxed]{algorithm2e}
\usepackage{mathtools, amsthm}
\usepackage{amsmath}
\usepackage{amssymb}
\usepackage{MnSymbol}
\usepackage{tikz}
\usepackage{enumitem}

\newtheorem{theorem}{Theorem}[section]

\newtheorem{lemma}[theorem]{Lemma}

\newtheorem{definition}[theorem]{Definition}

\newcommand{\fedexp}{{\sc FedExp3}}
\newcommand{\gucb}{{\sc GUCB}}

\newcommand*{\rom}[1]{%
  \textup{\uppercase\expandafter{\romannumeral#1}}%
}

\title{Doubly Adversarial Federated Bandits}

\date{} 					% Or removing it

\author{ \href{https://orcid.org/0000-0001-5592-8314}{Jialin Yi\includegraphics[scale=0.06]{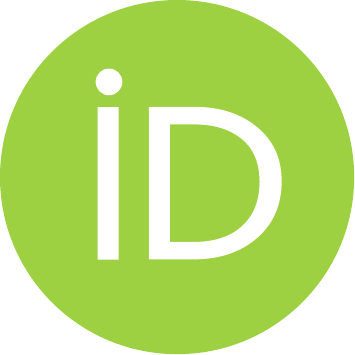}\hspace{1mm}}\quad Milan Vojnovic \\
	Department of Statistics\\
	London School of Economics and Political Science\\
	WC2A 2AE, London, United Kingdom \\
	\texttt{\{j.yi8,  m.vojnovic\}@lse.ac.uk} \\
}

\renewcommand{\headeright}{}
\renewcommand{\undertitle}{}
\renewcommand{\shorttitle}{\textit{Doubly Adversarial Federated Bandits}}

\hypersetup{
pdftitle={A template for the arxiv style},
pdfsubject={q-bio.NC, q-bio.QM},
pdfauthor={David S.~Hippocampus, Elias D.~Striatum},
pdfkeywords={First keyword, Second keyword, More},
}

\begin{document}
\maketitle

\begin{abstract}
We study a new non-stochastic federated multi-armed bandit problem with multiple agents collaborating via a communication network. The losses of the arms are assigned by an oblivious adversary that specifies the loss of each arm not only for each time step but also for each agent, which we call ``doubly adversarial". In this setting, different agents may choose the same arm in the same time step but observe different feedback. The goal of each agent is to find a globally best arm in hindsight that has the lowest cumulative loss averaged over all agents, which necessities the communication among agents.
We provide regret lower bounds for any federated bandit algorithm under different settings, when agents have access to full-information feedback, or the bandit feedback.
For the bandit feedback setting, we propose a near-optimal federated bandit algorithm called \fedexp. %FedExp3.
Our algorithm gives a positive answer to an open question proposed in \cite{cesa2016delay}: \fedexp\ can guarantee a sub-linear regret without exchanging sequences of selected arm identities or loss sequences among agents.
    % Compared to the prior work, FedExp3 keeps the raw data of agents decentralized stored hence applies to the federated learning systems.
We also provide numerical evaluations of our algorithm to validate our theoretical results and demonstrate its effectiveness on synthetic and real-world datasets. %MovieLens dataset. 
\end{abstract}

% keywords can be removed
% \keywords{First keyword \and Second keyword \and More}

\section{Introduction}
There is a rising trend of research on federated learning, which coordinates multiple \emph{heterogeneous} agents to collectively train a learning algorithm, while keeping the raw data decentralized \citep{kairouz2021advances}.
We consider the federated learning variant of a multi-armed bandit problem which is one of the most fundamental sequential decision making problems.
In standard multi-armed bandit problems, a learning agent needs to balance the trade-off between exploring various arms in order to learn how much rewarding they are and selecting high-rewarding arms.
In federated bandit problems, multiple heterogeneous agents collaborate with each other to maximize their cumulative rewards.
The challenge here is to design decentralized collaborative algorithms to find a \emph{globally} best arm for all %heterogeneous 
agents while keeping their raw data decentralized.

Finding a globally best arm with raw arm or loss sequences stored in a distributed system has ubiquitous applications in many systems built with a network of learning agents.
One application is in recommender systems where different recommendation app clients (i.e. agents) in a communication network collaborate with each other to find news articles (i.e. arms) that are popular among all users within a specific region, which can be helpful to solve the \emph{cold start} problem \citep{li2010contextual, yi2021efficient}.
In this setting, the system avoids the exchange of users' browsing history (i.e. arm or loss sequences) between different clients for better privacy protection.
Another motivation is in international collaborative drug discovery research, where different countries (i.e. agents) cooperate with each other to find a drug (i.e. arm) that is uniformly effective for all patients across the world \citep{varatharajah2022contextual}. To protect the privacy of the patients involved in the research, the exact treatment history of specific patients (i.e. arm or loss sequences) should not be shared during the collaboration.

The federated bandit problems are focused on identifying a globally best arm (pure exploration) or maximizing the cumulative  group reward (regret minimization) in face of heterogeneous feedback from different agents for the same arm, which has gained much attention in recent years \citep{dubey2020differentially, zhu2021federated, huang2021federated, shi2021federated, reda2022near}. In prior work, heterogeneous feedbacks received by different agents are modeled as samples from some unknown but fixed distributions.
Though this formulation of heterogeneous feedback allows elegant statistical analysis of the regret, it may not be adequate for dynamic (non-stationary) environments. For example, consider the task of finding popular news articles within a region mentioned above. The popularity of news articles on different topics can be time-varying, e.g. the news on football may become most popular during the FIFA World Cup but may be less popular afterwards.

In contrast with the prior work, we introduce a new \emph{non-stochastic} federated multi-armed bandit problem in which the heterogeneous feedback received by different agents are chosen by an oblivious adversary.
We consider a federated bandit problem with $K$ arms and $N$ agents. The agents can share their information via a communication network. At each time step, each agent will choose one arm, receive the feedback and exchange their information with their neighbors in the network. The problem is \emph{doubly adversarial}, i.e. the losses are determined by an oblivious adversary which specifies the loss of each arm not only for each time step but also for each agent. As a result, the agents which choose the same arm at the same time step may observe different losses. 
The goal is to find the \emph{globally} best arm in hindsight, whose cumulative loss averaged over all agents is lowest, without exchanging raw information consisting of arm identity or loss value sequences among agents.
% We will consider two variants of this problem: either the full-information feedback setting, i.e., agents has access to the losses of all arms, or the bandit feedback setting, i.e., agents only have access to the loss of arm they choose.
As standard in online learning problems, we focus on regret minimization over an arbitrary time horizon.

%\mv{Perhaps providing some more justification for why studying both full information and bandit feedback.}

\subsection{Related work}

The doubly adversarial federated bandit problem is related to several lines of research, namely that on federated bandits, multi-agent cooperative adversarial bandits, and distributed optimization. Here we briefly discuss these related works.

\paragraph{Federated bandits}

Solving bandit problems in the federated learning setting has gained attention in recent years. \cite{dubey2020differentially} and \cite{huang2021federated}  considered the linear contextual bandit problem and extended the LinUCB algorithm \citep{li2010contextual} to the federated learning setting.
\cite{zhu2021federated} and \cite{shi2021federated}
studied a federated multi-armed bandit problem where the losses observed by different agents are i.i.d. samples from some common unknown distribution.
\cite{reda2022near} considered the problem of identifying a globally best arm for multi-armed bandit problems in a centralized federated learning setting. All these works focus on the stochastic setting, i.e. the reward or loss of an arm is sampled from some unknown but fixed distribution. Our work considers the non-stochastic setting, i.e. losses are chosen by an oblivious adversary, which is a more appropriate assumption for non-stationary environments. 

\paragraph{Multi-agent cooperative adversarial bandit}
\cite{cesa2016delay, bar2019individual, yi2022regret} studied the adversarial case where agents receive the same loss for the same action chosen at the same time step, whose algorithms require the agents to exchange their raw data with neighbors.
\cite{cesa2020cooperative} discussed the cooperative online learning setting where the agents have access to the full-information feedback and the communication is asynchronous. In these works, the agents that choose the same action at the same time step receive the same reward or loss value and agents aggregate messages received from their neighbors. Our work relaxes this assumption by allowing agents to receive different losses even for the same action in a time step. Besides, we propose a new algorithm that uses a different aggregation of messages than in the aforementioned papers, which is based on distributed dual averaging method in \cite{nesterov2009primal, xiao2009dual, duchi2011dual}.

% \paragraph{Distributed online learning} 
%      $O(T^{3/4})$ regret upper bound appears in distributed online convex optimization \citep{wan2022projection} and distributed online linear regression \citep{DBLP:journals/tit/YuanPS21}. \mv{Here we mention $O(T^{3/4})$ while our new upper bounds are $O(T^{2/3})$?}

\paragraph{Distributed optimization} \cite{duchi2011dual} proposed the dual averaging algorithm for distributed convex optimization via a gossip communication mechanism. Subsequently,  \cite{hosseini2013online} extended this algorithm to the online optimization setting.
\cite{scaman2019optimal} found optimal distributed algorithms for distributed convex optimization and a lower bound which applies to strongly convex functions. The doubly adversarial federated bandit problem with full-information feedback is a special case of distributed online linear optimization problems. Our work complements these existing studies by providing a lower bound for the distributed online linear optimization problems. Moreover, our work proposes a near-optimal algorithm for the more challenging bandit feedback setting.

%\mv{May add punchlines for different groups of related works, discussing how our problem and/or results are different. }

\subsection{Organization of the paper and our contributions}

We first formally formulate the doubly adversarial federated bandit problem and the federated bandit algorithms we study in Section~\ref{sec::problem-setting}.
Then, in Section ~\ref{sec:lower-bound}, we provide two regret lower bounds for any federated bandit algorithm under the full-information and bandit feedback setting, respectively.
In Section~\ref{sec:upper-bound}, we present a federated bandit algorithm adapted from the celebrated Exp3 algorithm for the bandit-feedback setting, together with its regret upper bound.
Finally, we show results of our numerical experiments in Section~\ref{sec:numerical-experiments}.

Our contributions can be summarized as follows:
\begin{enumerate}[label=(\roman*)]
    \item We introduce a new federated bandit setting, doubly adversarial federated bandits, in which no stochastic assumptions are made for the heterogeneous losses received by the agents. This adversarial setting complements the prior work focuses on the stochastic setting.
    \item For both the full-information and bandit feedback setting, we provide regret lower bounds for any federated bandit algorithm. The regret lower bound for the full-information setting also applies to distributed online linear optimization problems, and, to the best of our knowledge, is the first lower bound result for this problem.
    \item For the bandit feedback setting, we propose a new near-optimal federated bandit Exp3 algorithm (\fedexp) with a sub-linear regret upper bound. Our \fedexp\ algorithm resolves an open question proposed in \cite{cesa2016delay}: it is possible to achieve a sub-linear regret for each agent simultaneously without exchanging both the action or loss information and the distribution information among agents.
\end{enumerate}

%\mv{May consider which results in the presented list are most significant, novel, and technically challenging and make sure there is some emphasis on them in what follows in the paper.}

\section{Problem setting} \label{sec::problem-setting}

Consider a communication network defined by an undirected graph $\mathcal{G}=(\mathcal{V}, \mathcal{E})$, where $\mathcal{V}$ is the set of $N$ agents and $(u,v)\in \mathcal{E}$ if agent $u$ and agent $v$ can directly exchange messages.
We assume that $\mathcal{G}$ is simple, i.e. it contains no self loops nor multiple edges.
The agents in the communication network collaboratively aim to solve a non-stochastic multi-armed bandit problem. In this problem, there is a fixed set $\mathcal{A}$ of $K$ arms and a fixed time horizon $T$.
%Without loss of generality, we assume $|\mathcal{V}|=N$ and $|\mathcal{A}|=K$.
Each instance of the problem is parameterized by a tensor $L = \left( \ell_t^v(i) \right)\in [0, 1]^{T\times N\times K}$ where $\ell_t^v(i)$ is the loss associated with agent $v\in \mathcal{V}$ if it chooses arm $i \in \mathcal{A}$ at time step $t$.

At each time step $t$, each agent $v$ will choose its action $a^v_t=i$, observe the feedback $I^v_t$ and incur a loss defined as the average of losses of arm $i$ over all agents, i.e., 
\begin{equation}
    \Bar{\ell}_t (i) = \frac{1}{N} \sum_{v\in \mathcal{V}} \ell_t^v(i).
\end{equation}
At the end of each time step, each agent $v\in \mathcal{V}$ can communicate with their neighbors $\mathcal{N}(v)=\{u\in \mathcal{V}: (u, v)\in \mathcal{E}\}$. 
We assume a \emph{non-stochastic} setting, i.e. the loss tensor $L$ is determined by an oblivious adversary. In this setting, the adversary has the knowledge of the description of the algorithm running by the agents but the losses in $L$ do not depend on the specific arms selected by the agents.

The performance of each agent $v\in \mathcal{V}$ is measured by its \emph{regret}, defined as the difference of the expected cumulative loss incurred and the cumulative loss of a \emph{globally} best fixed arm in hindsight, i.e.
\begin{equation}
R_T^v(\pi, L) = \mathbb{E} \left[ \sum_{t=1}^T \Bar{\ell}_t(a_t^v) - \min_{i\in \mathcal{A}}\left\{ \sum_{t=1}^T \Bar{\ell}_t(i)\right\} \right]
\label{equ:regret}
\end{equation}
where the expectation is taken over the action of all agents under algorithm $\pi$ on instance $L$. 
We will abbreviate $R_T^v(\pi, L)$ as $R_T^v$ when the algorithm  $\pi$ and instance $L$ have no ambiguity in the context.
We aim to characterize $\max_{L}R_T^v(\pi, L)$ for each agent $v\in \mathcal{V}$ under two feedback settings, 
\begin{itemize}
    \item full-information feedback: $I^v_t = \ell_t^v$, and
    \item bandit feedback: $I^v_t = \ell_t^v(a^v_t)$.
\end{itemize}
Let  $\mathcal{F}^v_t$ be the sequence of agent $v$'s actions and feedback up to time step $t$, i.e., $\mathcal{F}^v_t = \bigcup_{s=1}^t \{a^v_s, I^v_s \}$. For a graph $\mathcal{G}=(\mathcal{V}, \mathcal{E})$, we denote as $d(u, v)$ the number of edges of a shortest path connecting nodes $u$ and $v$ in $\mathcal{V}$ and $d(v,v)=0$.

We focus on the case when $\pi$ is a \emph{federated bandit algorithm} in which each agent $v\in\mathcal{V}$ can only communicate with their neighbors within a time step.
\begin{definition}[federated bandit algorithm]\label{def:federated}
    A federated bandit algorithm $\pi$ is a multi-agent learning algorithm such that for each round $t$ and each agent $v\in \mathcal{V}$, the action selection distribution $p^v_t$ only depends on $\bigcup_{u\in \mathcal{V}} \mathcal{F}^u_{t-d(u, v)-1}$.
\end{definition}
From Definition~\ref{def:federated}, the communication between any two agents $u$ and $v$ in $\mathcal{G}$ comes with a delay equal to $d(u, v)+1$. Here we give some examples of $\pi$ in different settings:
\begin{itemize}
    \item when $|\mathcal{V}| = 1$ and $I^v_t = \ell_t^v$, $\pi$ is an online learning algorithm for learning with expert advice problems \citep{cesa1997use},
    \item when $|\mathcal{V}| = 1$ and $I^v_t = \ell_t^v(a^v_t)$, $\pi$ is a sequential algorithm for a multi-armed bandit problem \citep{auer2002finite}, 
    \item when $I^v_t \in \partial f_i(x^v_t)$ for some convex function $f(x)$, $\pi$ belongs to the black-box procedure for distributed convex optimization over a simplex \citep{scaman2019optimal}, and
    \item when $\mathcal{G}$ is a star graph, $\pi$ is a centralized federated bandit algorithm discussed in \cite{reda2022near}.
\end{itemize}

% \mv{Circle notation appears strange to me as it suggests it is a binary operator. Maybe write instead $\mathbb{E}_{\pi,L}$, if it is necessary to indicate $\pi$ and $L$.}

% \subsection{Black-box online learning procedure}

\section{Lower bounds} \label{sec:lower-bound}

In this section, we show two lower bounds on the cumulative regret of any federated bandit algorithm $\pi$ in which all agents exchange their messages through the communication network $\mathcal{G} = \left(\mathcal{V}, \mathcal{E}\right)$, for full-information and bandit feedback setting. 
Both lower bounds highlight how the cumulative regret of the federated algorithm $\pi$ is related to the minimum time it takes for all agents in $\mathcal{G}$ to reach an agreement on a globally best arm.
% \mv{May reconsider the structure of sections. Currently, there is first a section for full-information feedback and then another section for full-bandit feedback. For the former, we only provide a lower bound, while for the latter we provide a lower and an upper bound. This appears strange. Maybe combining all theoretical results in one section? Also would need to explain why for the full information feedback we provide only a lower bound.}

Agents reaching an agreement about a globally best arm in hindsight is to find $i^\ast \in \arg\min_{i\in \mathcal{A}} \sum_{t=1}^T \Bar{\ell}_t(i)$ by each agent $v$ exchanging their private information about $\{\sum_{t=1}^T \ell^v_t(i): i\in \mathcal{A}\}$ with their neighbors. 
This is known as a distributed consensus averaging problem \citep{boyd2004fastest}. 
% \mv{The last sentence is ambiguous because reaching an agreement by exchanging information with neighbors can be for various kinds of agreement objectives while distributed consensus averaging is for a particular agreement objective.} 
% \mv{*** How about using notation $d_u$, $d_{\min}$ and $d_{\max}$ instead of $d_u$, $d_{\min}$ and $d_{\max}$? This would be more standard notation. ***} 
Let $d_v = |\mathcal{N}(v)|$ and $d_{\max} = \max_{v\in\mathcal{V}} d_v$ and $d_{\min} = \min_{v\in\mathcal{V}} d_v$ .
The dynamics of a consensus averaging procedure is usually characterized by spectrum of the Laplacian matrix $M$ of graph $\mathcal{G}$ defined as
$$
M_{u, v}:= \begin{cases} d_{u} & \text { if } u=v \\ -1 & \text { if } u \neq v \text { and } (u,v)\in \mathcal{E} %u \text { is adjacent to } v 
\\ 0 & \text { otherwise}.\end{cases}
$$
% a family of matrices $\mathcal{M}_\mathcal{G}$.
% Let $\mathcal{M}_\mathcal{G}$ be the set of all $N\times N$ square real matrix $M$ such that
% \begin{enumerate}
%     \item $M$ is symmetric and positive semi-definite;
%     \item the kernel of $M$ is the set of constant vectors, i.e. $\{x\in \mathbb{R}^N: Mx=0\} = \operatorname{Span}(\mathbb{1})$ where $\mathbb{1}=(1, \ldots, 1)^{\top}$;
%     \item $M_{i,j} \neq 0$ only if $i=j$ or $(i,j)\in \mathcal{E}$.
% \end{enumerate}
Let $\lambda_1(M)\geq \cdots\geq \lambda_N(M) = 0$ be the eigenvalues of the Laplacian matrix $M$. 
% \mv{*** Hereinafter, it should be $N$ instead of $n$, as we use capital letters to denote number agents, number of arms, and horizon time. ***}
% The convergence analysis of a consensus procedure usually depends on $\jy{\frac{1+d_{\max}}{\lambda_{N-1}(M)} =} \lambda_1(M) / \lambda_{N-1}(M)$ \citep{scaman2019optimal}, 
% \mv{Why is this called an \emph{eigengap} while being defined as a ratio of eigenvalues?}
% We denote $\frac{1+d_{\max}}{\lambda_{N-1}(M)} = \min_{M\in\mathcal{M}_\mathcal{G}} \gamma(M)$ 
% which measures the time it takes for all agents in the communication network $\mathcal{G}$ to reach a consensus using distributed averaging algorithms.
The second smallest eigenvalue $\lambda_{N-1}(M)$ is the \emph{algebraic connectivity} which approximates the sparest-cut of graph $\mathcal{G}$ \citep{arora2009expander}.
In the following two theorems, we show that for any federated bandit algorithm $\pi$, there always exists a problem instance and an agent whose worst-case cumulative regret is $\Omega(\lambda_{N-1}(M)^{-1/4} \sqrt{T})$.

\begin{theorem}[Full-information feedback]
    \label{thm:lower-bound-full}For any federated bandit algorithm $\pi$, there exists a graph $\mathcal{G}=(\mathcal{V}, \mathcal{E})$ with Laplacian matrix $M$ and a full-information feedback instance $L\in [0, 1]^{T\times N\times K}$ such that for some $v_1 \in \mathcal{V}$,
    \begin{equation}
        \label{lower-bound:full-feedback}
        R^{v_1}_T(\pi, L) = \Omega \left(\sqrt[4]{\frac{1+d_{\max}}{\lambda_{N-1}(M)}}\sqrt{T\log K}\right).
    \end{equation}
\end{theorem}

The proof, in Appendix~\ref{app-lower-full}, relies on the existence of a graph in which there exist two clusters of agents, $A$ and $B$, with distance $d(A, B) = \min_{u\in A, v\in B}d(u, v) = \Omega\left(\sqrt{(d_{\max}+1)/\lambda_{N-1}(M)}\right)$. Then, we consider an instance where only agents in cluster $A$ receive non-zero losses. 
Based on a reduction argument, the cumulative regrets for agents in cluster $B$ are the same as (up to a constant factor) the cumulative regret in a single-agent adversarial bandit problem with feedback of delay $d(A, B)$ (see Lemma~\ref{lm:delay} in Appendix~\ref{app-lower-full}). Hence, one can show that the cumulative regret of agents in cluster $B$ is $\Omega \left(\sqrt{d(A,B)}\sqrt{T\log K}\right)$. 

Note that the doubly adversarial federated bandit with full-information feedback is a special case of distributed online linear optimization, with the decision set being a $K-1$-dimensional simplex. Hence, Theorem~\ref{thm:lower-bound-full} immediately implies a regret lower bound for the distributed online linear optimization problem. To the best of our knowledge, this is the first non-trivial lower bound that relates the hardness of distributed online linear optimization problem to the algebraic connectivity of the communication network.

Leveraging the lower bound for the full-information setting, we show a lower bound for the bandit feedback setting.

\begin{theorem}[Bandit feedback]
    \label{thm:lower-bound-bandit}
    For any federated bandit algorithm $\pi$, there exists a graph $\mathcal{G}=(\mathcal{V}, \mathcal{E})$ with Laplacian matrix $M$ and a bandit feedback instance $L\in [0, 1]^{T\times N\times K}$ such that for some $v_1 \in \mathcal{V}$,
    \begin{equation}
        \label{lower-bound:bandit-feedback}
        R^{v_1}_T(\pi, L)= \Omega\left( \max\left\{ \sqrt{ \frac{1}{1+d_{v_1}}}\sqrt{KT}, \sqrt[4]{\frac{1+d_{\max}}{\lambda_{N-1}(M)} } \sqrt{T\log K} \right\} \right).
    \end{equation}
% \mv{Should $v$ in the right-hand side of the last equation be $v_1$?}
\end{theorem}

The proof is provided in Appendix~\ref{app:lower-bandit}. The lower bound contains two parts. The first part, derived from the information-theoretic argument in \cite{shamir2014fundamental},  captures the effect from bandit feedback. The second part is inherited from Theorem~\ref{thm:lower-bound-full} by the fact that the regret of an agent in bandit feedback setting cannot be smaller than its regret in full-information setting.

\section{\fedexp: a federated regret-minimization algorithm} 
\label{sec:upper-bound}

% \mv{*** Is qualification "near-optimal" justified given how  upper bound compares with lower bound? ***}
% \jy{*** In the discussion paragraph, we show that the ratio between regret upper bound of \fedexp\ and the lower bound is $\Tilde{O}(T^{1/6})$, I suppose this can justify the "near-optimality" of \fedexp.***} \mv{*** How about dependence on other parameters? In general, there is a danger in trying to oversell results -- I am not saying this is the case here, but better be careful with qualification of results. Maybe not stating "near-optimal" in the section title, but saying it in the text where appropriate discussion is provided. ***}

Inspired by the fact that the cumulative regret is related to the time need to reach consensus about a globally best arm,  we introduce a new federated bandit algorithm based on the gossip communication mechanism, called \fedexp. The details of \fedexp\ are described in Algorithm~\ref{algo:FedExp3}. We shall also show that \fedexp\ has a sub-linear cumulative regret upper bound which holds for all agents simultaneously.

The \fedexp\ algorithm is adapted from the Exp3 algorithm, in which each agent $v$ maintains an estimator $z^v_t \in \mathbb{R}^K$ of the cumulative losses for all arms and a tentative action selection distribution  $x^v_t\in[0, 1]^K$.
At the beginning of each time step $t$, each agent follows the action selection distribution $x^v_t$ with probability $1-\gamma_t$, or performs a uniform random exploration with probability $\gamma_t$. Once the action $a^v_t$ is sampled, the agent observes the associated loss $\ell^v_t(a^v_t)$ and then computes an importance-weighted loss estimator $g^v_t  \in \mathbb{R}^K$ using the sampling probability $p^v_t(a^v_t)$. Before $g^v_t$ is integrated into the cumulative loss estimator $z^v_{t+1}$, the agent communicates with its neighbors to average its cumulative loss estimator $z^v_t$.
% In this section, we present the \fedexp\ algorithm,  \mv{Provide some discussion for the definition of the algorithm.}
%\subsection{\fedexp\ algorithm}
\begin{algorithm}[h]
\caption{\fedexp}
\label{algo:FedExp3}
\SetKwInOut{Init}{Initialization}
\SetKwInOut{Input}{Input}
\Input{Non-increasing sequence of learning rates $\{\eta_t>0\}$ , non-increasing sequence of exploration ratios $\{\gamma_t>0\}$, and a gossip matrix $W\in [0, 1]^{N\times N}$.} 
\Init{$z_1^v(i) = 0, x^v_1(i)=1/K$ for all $i\in \mathcal{A}$ and $v\in \mathcal{V}$.}
\For{each time step $t=1,2,\dots, T$}{
    \For{each agent $v\in \mathcal{V}$}{
        compute the action distribution $p_t^v(i) = (1-\gamma_t) x^v_t(i) + \gamma_t/K$\;
        choose the action $a_t^v\sim p_t^{v}$\;
        compute the loss estimators
        $
        g^v_t(i) = \ell^v_t(i) \mathbb{I}\left\{a^v_t = i \right\} / p^v_t(i) 
        $ for all $i\in \mathcal{A}$\;
        update the gossip accumulative loss $z_{t+1}^v = \sum_{u: (u,v)\in \mathcal{E}} W_{u,v} z_t^u + g_t^v$\;
        update the exploitation distribution $x^v_{t+1} = \frac{\exp\{-\eta_t z_{t+1}^v(i)\}}{\sum_{j\in A} \exp\{-\eta_t z_{t+1}^v(j)\}}$\;
    }
}
\end{algorithm}

The communication step is characterized by the \emph{gossip} matrix which is a doubly stochastic matrix $W\in [0, 1]^{N\times N}$ satisfying the following constraints
\begin{equation*}
    \sum_{v\in \mathcal{V}}W_{u,v} = \sum_{u\in \mathcal{V}}W_{u,v} = 1 \quad \text{and} \quad W_{u,v} \geq 0 \quad\text{where equality holds when } (u,v)\notin \mathcal{E}.  
\end{equation*}
This gossip communication step facilitates the agents to reach a consensus on the estimators of the cumulative losses of all arms, and hence allows the agents to identify a globally best arm in hindsight.
We present below an upper bound on the cumulative regret of each agent in \fedexp.
\begin{theorem}
\label{thm:upper-bound:static}
Assume that the network runs Algorithm~\ref{algo:FedExp3} with
$$
\gamma_t = \sqrt[3]{\frac{\left(C_W + \frac{1}{2}\right)K^2\log K}{t}}
\quad
\text{and}
\quad
\eta_t = \frac{\log K}{T \gamma_T} = \sqrt[3]{\frac{(\log K)^2}{ \left(C_W + \frac{1}{2}\right) K^2 T^2}}
$$
with
$
C_W = \min\{2\log T + \log N, \sqrt{N}\}/(1-\sigma_2(W)) +3.
$
Then, the expected regret of each agent $v\in \mathcal{V}$ is bounded as
$$
R_T^v = \Tilde{O} \left( \frac{1}{\sqrt[3]{1-\sigma_2(W)}} K^{2/3} T^{2/3}\right) 
$$
%$$
%R_T^v = \Tilde{O} \left( \frac{1}{\sqrt[3]{1-\sigma_2(W)}} \sqrt[3]{K^2\log K} T^{\frac{2}{3}}\right) 
%$$
where $\sigma_2(W)$ is the second largest singular value of $W$. 
% \mv{In the theorem, $\tilde{O}(\cdot)$ would hide some poly-log factors, but we write $\log(K)^{1/3}$ therein?}
\end{theorem}

\paragraph{Proof sketch}
Let $\hat{\ell}_t$ and $\bar{z}_t$ be the average instant loss estimator and average cumulative loss estimator, 
\begin{equation*}
    f_t = \frac{1}{N}\sum_{v\in \mathcal{V}} g^v_t \quad \text{and} \quad \bar{z}_t = \frac{1}{N}\sum_{v\in \mathcal{V}} z^v_t,
\end{equation*}
and let $y_t$ be action distribution that minimizes the regularized average cumulative loss estimator
$$
y_t(i) = \frac{\exp\{-\eta_{t-1} \bar{z}_t(i)\}}{\sum_{j\in A} \exp\{-\eta_{t-1} \bar{z}_t(j)\}}.
$$
The cumulative regret can be bounded by the sum of three terms
\begin{equation*}
    R^v_t \leq 
    \underbrace{\mathbb{E}\left[\sum_{t=1}^T \left(\langle f_t, y^v_t \rangle - f_t (i^\ast) \right)\right]}_{\operatorname{FTRL}} + 
      K \underbrace{\sum_{t=1}^T \eta_{t-1} \mathbb{E}\|z^{v}_t - \bar{z}_t\|_\ast}_{\operatorname{CONSENSUS}} + \underbrace{\sum_{t=1}^T \gamma_t}_{\operatorname{EXPLORATION}}
\end{equation*}
where $i^\ast \in \arg\min_{i\in \mathcal{A}} \sum_{t=1}^T \Bar{\ell}_t(i)$ is a globally best arm in hindsight.

The $\operatorname{FTRL}$ term is a typical regret term from the classic analysis for the Follow-The-Regularized-Leader algorithm \citep{lattimore_szepesvari_2020}. The $\operatorname{CONSENSUS}$ term measures the cumulative approximation error generated during the consensus reaching process, which can be bounded using the convergence analysis of distributed averaging algorithm based on doubly stochastic matrices \citep{duchi2011dual, hosseini2013online}. The last $\operatorname{EXPLORATION}$ term can be bounded by specifying the time-decaying exploration ratio $\gamma_t$.
The full proof of Theorem~\ref{thm:upper-bound:static} is provided in Appendix~\ref{proof:upper-bound}.

The \fedexp\ algorithm is also a valid algorithm for the multi-agent adversarial bandit problem \citep{cesa2016delay} which is a special case of the doubly adversarial federated bandit problem when $\ell^v_t(i) = \ell_t(i)$ for all $v\in \mathcal{V}$. According to the distributed consensus process of \fedexp, each agent $v\in \mathcal{V}$ only communicates cumulative loss estimator values $z^v_t$, instead of the actual loss values $\ell^v_t(a^v_t)$, and the selection distribution $p^v_t$.
\fedexp\ can guarantee a sub-linear regret without the exchange of sequences of selected arm identities or loss sequences of agents, which resolves an open question raised in \cite{cesa2016delay}.

\paragraph{Choice of the gossip matrix} 
The gossip matrix $W$ can be constructed using the \emph{max-degree} trick in \cite{duchi2011dual}, i.e.,
$$
W  = I-\frac{D-A}{2(1+d_{\max})}
$$
where $D = \operatorname{diag}(d_1, \dots, d_N)$ and $A$ is the adjacency matrix of $\mathcal{G}$. This construction of $W$ requires that all agents have knowledge of the maximum degree $d_{\max}$, which can indeed be easily computed in a distributed system by nodes exchanging messages and updating their states using the maximum reduce operator.

% Before running the \fedexp\ algorithm, the agents can run a distributed averaging algorithm \citep{DBLP:journals/tit/BoydGPS06} to reach an consensus on $d_{\max}$. \mv{*** Why referring to this averaging algorithm for computing maximum value? This is a distributed \emph{selection problem}. It can indeed be easily computed in a distributed system by nodes exchanging messages and updating their states using maximum reduce operator. ***}

Another choice of $W$ comes from the effort to minimize the cumulative regret.
The leading factor $1/\sqrt[3]{1-\sigma_2(W)}$ in the regret upper bound of \fedexp\ can be minimized by choosing a gossip matrix $W$ with smallest $\sigma_2(W)$. Suppose that the agents have knowledge of the topology structure of the communication graph $\mathcal{G}$, then the agents can choose the gossip matrix to minimize their regret by solving the following convex optimization problem:
\begin{equation}
\begin{array}{rl}
\operatorname{minimize} & \left\|W-(1 / n) \mathbf{1 1}^T\right\|_2 \nonumber \\
\text { subject to} & W \geq 0, W \mathbf{1}=\mathbf{1}, W=W^T, \hbox{ and } W_{i j}=0, \hbox{ for } (i, j) \notin \mathcal{E}  \end{array}
\end{equation}
which has an equivalent semi-definite programming formulation as noted in \cite{boyd2004fastest}.
% Let $\Sigma$ be the Laplacian matrix of $\mathcal{G}$, \mv{Need to define the Laplacian matrix as there are several different definitions. Unclear why Laplacian matrix is mentioned in this context.} then $\Sigma\in \mathcal{M}_\mathcal{G}$.

\paragraph{Gap between upper and lower bounds} 
 The regret upper bound of \fedexp\ algorithm in Theorem~\ref{thm:upper-bound:static} grows sublinearly in the number of arms $K$ and horizon time $T$. There is only a small polynomial gap between the regret upper bound and the lower bound in Theorem~\ref{thm:lower-bound-bandit} with respect to these two parameters. The regret upper bound depends also on the second largest singular value $\sigma_2(W)$ of $W$. The related term in the lower bound in Theorem~\ref{thm:lower-bound-bandit} is the second smallest eigenvalue $\lambda_{N-1}(M)$ of the Laplacian matrix $M$. To compare these two terms, we point that 
 when the gossip matrix is constructed using the max-degree method, as discussed in Corollary 1 in \cite{duchi2011dual},
 $$ \frac{1}{\sqrt[3]{1- \sigma_2(W)}} \leq \sqrt[3]{2\frac{d_{\max }+1}{\lambda_{N-1}(M)}}.$$ 
 % \mv{*** Hereinafter, why sometimes writing $x^{1/r}$ and in other cases writing $\sqrt[r]{x}$? Maybe use one or the other consistently throughout in the paper. ***}
 With respect to $\sqrt[4]{(d_{\max}+1)/\lambda_{N-1}(M)}$ in Theorem~\ref{thm:lower-bound-bandit}, there is only a small polynomial gap between the regret upper bound and the lower bound. 
 % \mv{May consider how to best slightly rephrase the last sentence so that it is clear. From the last above inequality, we have $(1/(1-\sigma_2(W)))^{1/3} \leq 2^{1/3} ((d_{\max}+1)/\lambda_{N-1}(M))^{1/3}$. Maybe it is sufficient to rephrase by saying " ...".}
 We note that a similar dependence on $\sigma_2(W)$ is present in the analysis of distributed optimization algorithms  \citep{duchi2011dual, hosseini2013online}.
 
\section{Numerical experiments}\label{sec:numerical-experiments}

We present experimental results for the \fedexp\ algorithm ($W$ constructed by the max-degree method) using both synthetic and real-world datasets. 
We aim to validate our theoretical analysis and demonstrate the effectiveness of \fedexp\ on finding a globally best arm in non-stationary environments. All the experiments are performed with 10 independent runs. The code for producing our experimental results is available online in the Github repository: \href{}{[link]}.

\subsection{Synthetic datasets}

We validate our theoretical analysis of the \fedexp\ algorithm on synthetic datasets. The objective is two-fold. First, we demonstrate that the cumulative regret of \fedexp\ grows sub-linearly with time. Second, we examine the dependence of the regret on the second largest singular value of the gossip matrix. %is as stated in Theorem~\ref{thm:upper-bound:static}.

A motivation for finding a globally best arm in recommender systems is to provide recommendations for those users whose feedback is sparse.
In this setting, we construct a federated bandit setting in which a subset of agents will be activated at each time step and only activated agents may receive non-zero loss.
Specifically, we set $T=3,000$ with $N=36$ and $K=20$. At each time step $t$, a subset $U_t$ of $N/2$ agents are selected from $\mathcal{V}$ without replacement.
For all activated agents $U_t$, the loss for arm $i$ is sampled independently from Bernoulli distribution with mean $\mu_i = (i-1)/(K-1)$. All non-activated agents receive a loss of $0$ for any arm they choose at time step $t$.

We evaluate the performance of \fedexp\ on different networks, i.e. for a complete graph, a $\sqrt{N}$ by $\sqrt{N}$ grid network, and random geometric graphs. 
The random geometric graph RGG($d$) is constructed by uniform random placement of each node in $[0, 1]^2$ and connecting any two nodes whose distance is less or equal to $d$ \citep{penrose2003random}. %In the context of geographic-based federated recommender systems, it is reasonable to assume that each client can only communicate with the clients that are located close to it. 
%Hence the random geometric graphs,  used to model the connectivity of the mobile device that can communicate with in some fixed radius, are suitable for this setting.
Random geometric graphs are commonly used for modeling spatial networks.
% The Erdos-Renyi graph ER($p$) \mv{*** Hereinafter (including references), it would be good to properly write Erd\H{o}s-R\' enyi. ***}, constructed by connecting any pair of agents with a fixed probability $p$, is also studied extensively in the context of distributed optimization \citep{frieze2016introduction, lei2020online}.
% Figure~ \jy{***Plots to be made***} illustrates the topology of the four kinds of networks. \mv{*** The justification for considering given graphs in the application setting considered may not be fully satisfactory. \cite{gupta2000capacity} is an information-theory paper on capacity of wireless networks, assuming a model which may not be motivated by real-world applications. Saying that Erdos-Renyi random graphs were used in some distributed optimization papers may not serve as a good justification for us to use in a setting we try to motivate by recommender system applications. One may ask: what graph should we assume in our setting? This should be a graph that captures some properties of real-world graphs in the context of recommender systems or other online services. In this view, it would make sense to consider a power-law graph, as many graphs in nature have a power-law degree distribution. Perhaps including a power-law graph among those covered already? ***}

\begin{figure}[t]
    \centering
    \includegraphics[width=.4\textwidth]{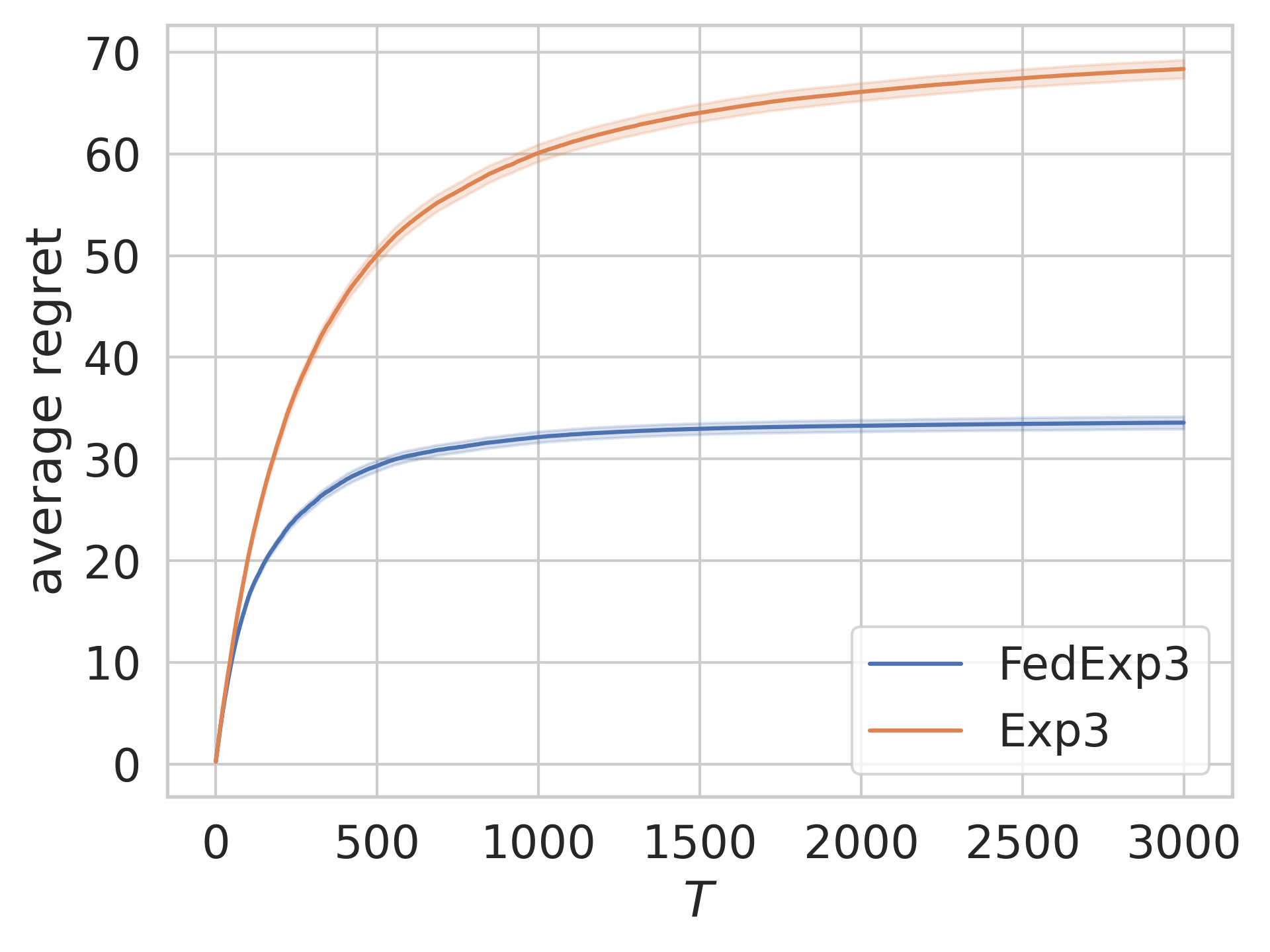}\hspace*{0.5cm}
    \includegraphics[width=.4\textwidth]{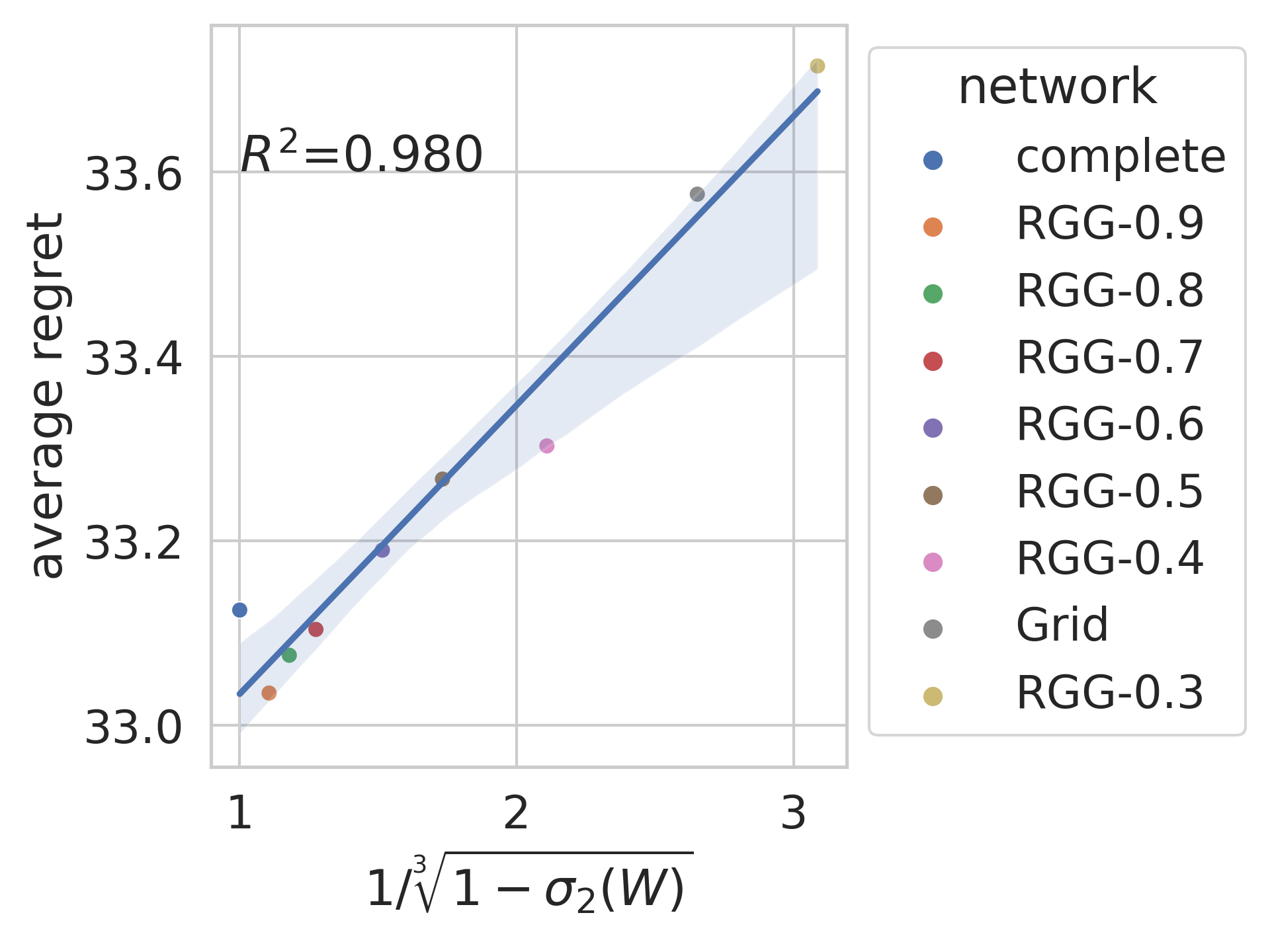}
    \caption{(Left) The average cumulative regret, i.e. $\sum_{v\in\mathcal{V}} R^v_T / N$, versus $T$, for \fedexp\ and Exp3 on the grid graph. (Right) The average cumulative regret versus $(1-\sigma_2(W))^{-1/3}$ for \fedexp\ on different networks at $T=3000$. %Each circle represents a different communication network. The blue line indicates the linear regression. 
    % \mv{*** Could $x$ and $y$ labels be shown in all plots, and presented in a standard way (not within the axis in the positive quadrant)? Could we not use titles for plots, but use legends when needed? ***}\jy{***The current plots are generated using Weights and Bias which is not the version we will include in our paper. I shall elaborate the plots later but for now they are good indicators.***}
    }
    \label{fig:regret}
\end{figure}

In our experiments, we set $d\in \{0.3, \dots, 0.9\}$. The results in 
Figure~\ref{fig:regret} confirm that the cumulative regret of the \fedexp\ algorithm grows sub-linearly with respect to time and suggest that the cumulative regret of \fedexp\ is proportional to $\left(1-\sigma_2(W)\right)^{-1/3}$. This is compatible with the regret upper bound in Theorem~\ref{thm:upper-bound:static}. 

\subsection{MovieLens dataset: recommending popular movie genres}

% \mv{*** What is the goal of this experimental analysis? Figure~\ref{fig:movie-lens} shows regret curves for different networks --- what do we want to show with this? Is there something we want to show that is of interest in the specific context of movie recommendations? ***}
We compare \fedexp\ with a UCB-based federated bandit algorithm in a movie recommendation scenario using a real-world dataset. In movie recommendation settings, users' preferences over different genres of movies can change over time. In such non-stationary environments, we demonstrate that a significant performance improvement can be achieved by \fedexp\ against the GossipUCB algorithm (we refer to as \gucb) proposed in \cite{zhu2021federated}, which is defined for stationary settings.

We evaluate the two algorithms using the \href{https://grouplens.org/datasets/movielens/latest/}{MovieLens-Latest-full} dataset which contains 58,000 movies, classified into 20 genres, with 27,000,000 ratings (rating scores in $\{0.5, 1,\ldots, 5\}$) from 280,000 users. Among all the users, there are 3,364 users who rated at least one movie for every genre. We select these users as our agents, i.e. $N = 3,364$, and the 20 genres as the arms to be recommended, i.e. $K = 20$.

We create a federated bandit environment for movie recommendation based on this dataset. 
Let $m^v(i)$ be the number of ratings that agent $v$ has for genre $i$. We set the horizon $T= \max_{v\in \mathcal{N}} m^v(i) =12,800 $.
To reflect the changes in agents' preferences
over genres as time evolves, we sort the ratings in an increasing order by their Unix timestamps and construct the loss tensor in the following way.
Let $r^v_j (i)$ be the $j$-th rating of agent $v$ on genre $i$, the loss of recommending an movie to agent $v$ of genre $i$ at time step $t$ is defined as 
    \begin{equation*}
        \ell^v_t(i) = \frac{5.5-r^v_j (i)}{5.5} \quad \text{for} \quad t \in \left[(j-1)\left\lfloor \frac{T}{m^v(i)} \right\rfloor, j\left\lfloor \frac{T}{m^v(i)} \right\rfloor\right).
    \end{equation*}
% Hence the doubly adversarial federated bandit problem is to recommend each agent the averaged best genre, e.g. \textit{Film Noir}. 
The performance of \fedexp\ and \gucb\ is shown in Figure~\ref{fig:movie-lens}. 
The results demonstrate that \fedexp\ can outperform \gucb\ by a significant margin for different communication networks.

\begin{figure}[t]
    \centering
    \includegraphics[width=.3\textwidth]{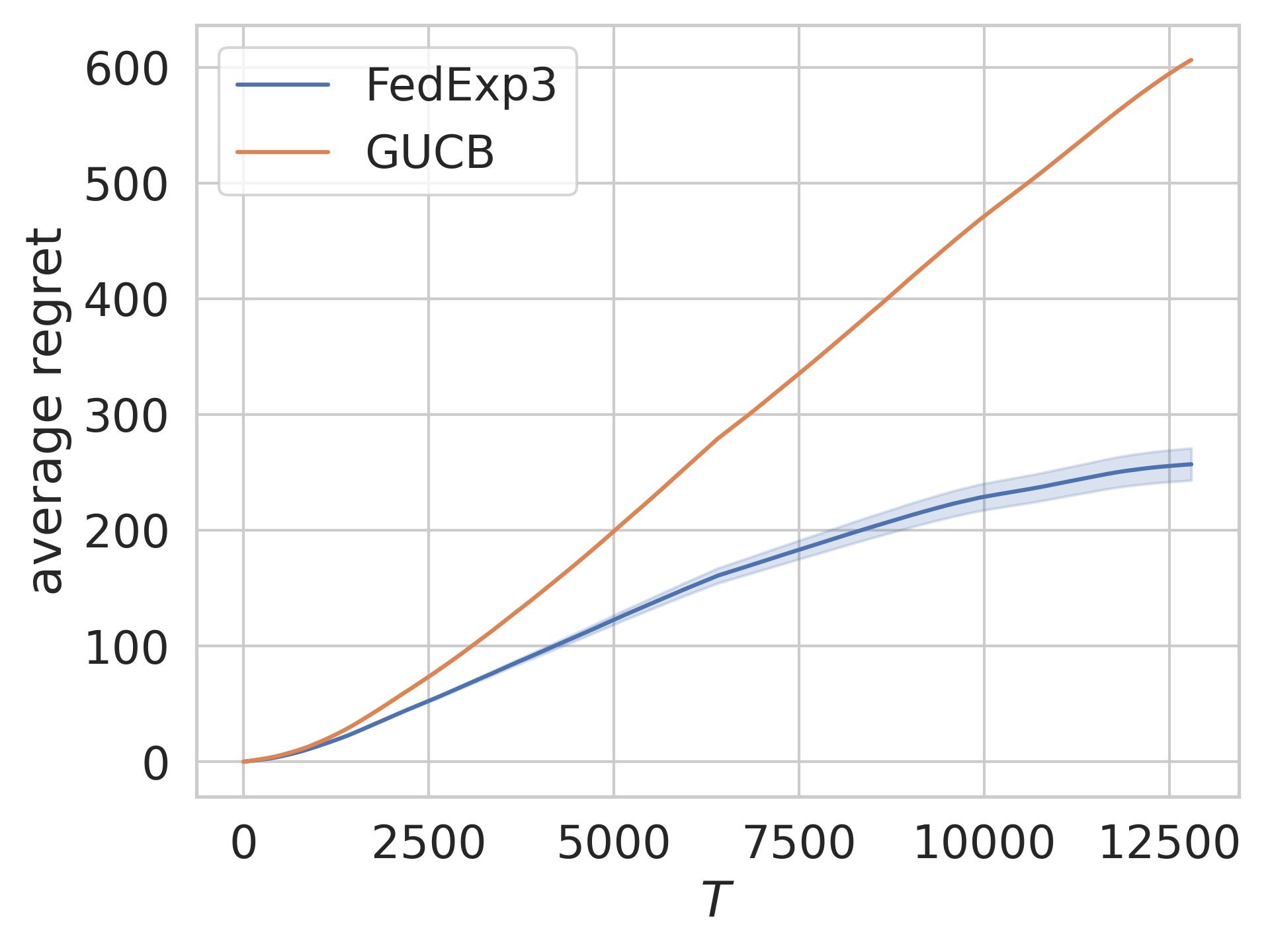}
    \includegraphics[width=.3\textwidth]{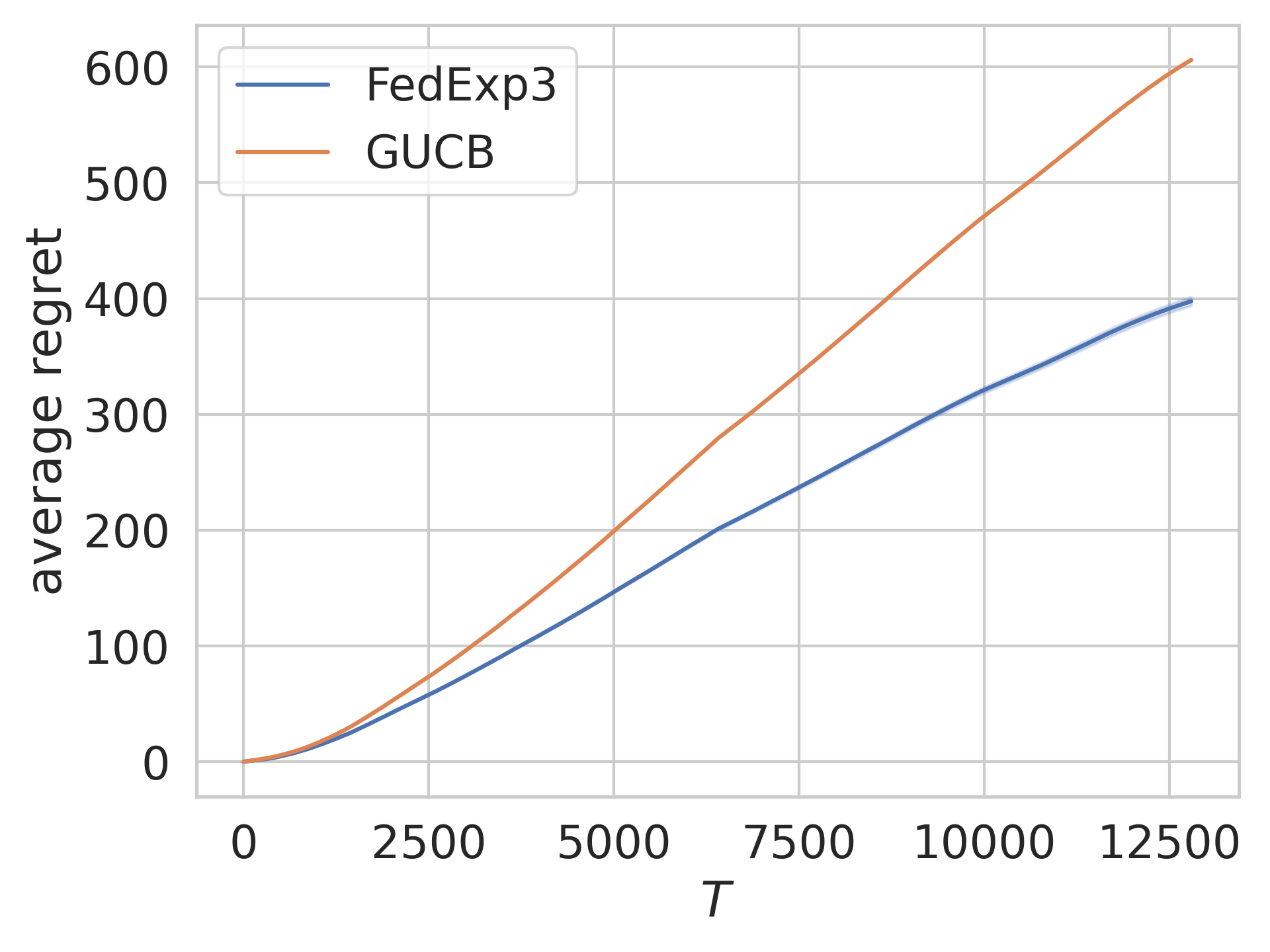}
    \includegraphics[width=.3\textwidth]{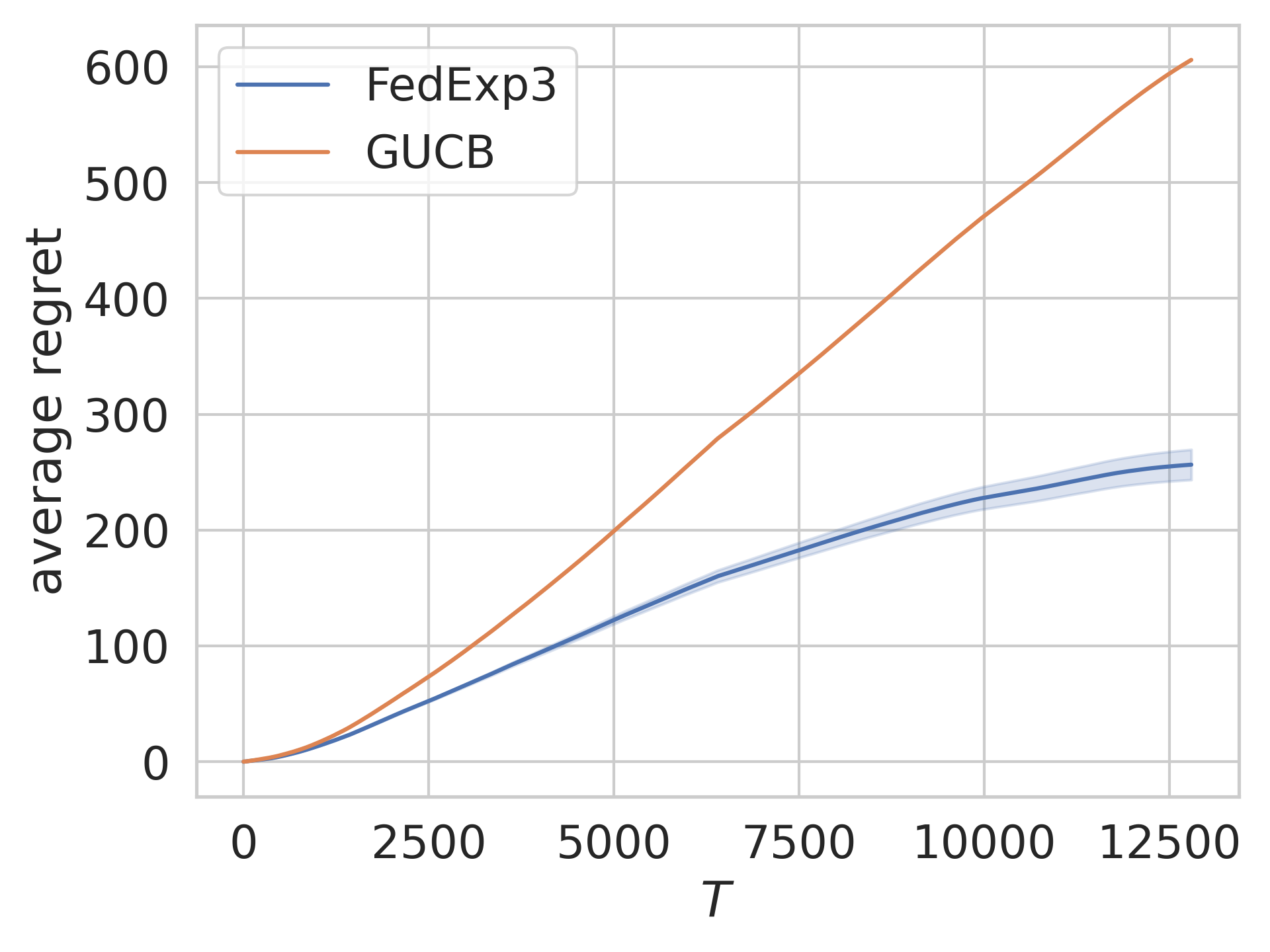}
    \caption{The average cumulative regret versus horizon time for \fedexp\ and \gucb\ in the movie recommendation setting with the communication networks: (left) complete graph, (mid) the grid network, and (right) RGG($0.5$). 
    % \mv{*** Same comments for $x$ and $y$ labels and plot titles as for Figure~1. ***}
    }
    \label{fig:movie-lens}
\end{figure}

\section{Conclusion and future research}

We studied doubly adversarial federated bandits, a new adversarial (non-stochastic) setting for federated bandits, which complement prior study on stochastic federated bandits. Firstly, we derived regret lower bounds for any federated bandit algorithm when the agents have access to full-information or bandit feedback.
These regret lower bounds relate the hardness of the problem to the algebraic connectivity of the network through which the agents communicate. Then we proposed the \fedexp\ algorithm which is a federated version of the Exp3 algorithm. We showed that there is only a small polynomial gap between the regret upper bound of \fedexp\ and the lower bound. Numerical experiments performed by using both synthetic and real-word datasets demonstrated that \fedexp\ can outperform the state-of-the-art stochastic federated bandit algorithm by a significant margin in non-stationary environments.

We point out some interesting avenues for future research on doubly adversarial federated bandits. The first is to close the gap between the regret upper bound of \fedexp\ algorithm and the lower bounds shown in this paper. The second is to extend the doubly adversarial assumption to federated linear bandit problems, where the doubly adversarial assumption could replace the stochastic assumption on the noise in the linear model.

\bibliographystyle{unsrtnat}
\bibliography{references} 

\newpage
\appendix

\section{Appendix}

% \mv{What is given next is not only notation but definitions. Definitions like that of the Fenchel conjugate may be presented later in the proof where they are actually used. If a notation or definition is not used throughout, then may be better define it locally where it is needed.}

\paragraph{Notation} For a vector $x$, we use $x(i)$ to denote the $i$-th coordinate of $x$. We define $\mathcal{F}_{t} = \bigcup_{v\in \mathcal{V}} \mathcal{F}^v_{t}$ where $\mathcal{F}^v_t$ is the sequence of agent $v$'s actions and feedback up to time step $t$, i.e., $\mathcal{F}^v_t = \bigcup_{s=1}^t \{a^v_s, I^v_s \}$. 
% \mv{Recall definition of $\mathcal{F}^v_t$.}

% \mv{The concept of a strongly convex function is used without definition or reference. Why \emph{strong} convexity is used to define the Fenchel conjugate?}

\subsection{Proof of Theorem~\ref{thm:lower-bound-full}}
\label{app-lower-full}

We first define a new class of cluster-based distributed online learning procedure, referred to as \emph{cluster-based federated algorithms},  in which the delay only occurs when the communication is between different clusters.
The regret lower bound for federated bandit algorithms will be no less than the regret lower bound for cluster-based federated algorithms, as shown in Lemma~\ref{lm:monotone}. Then we show in Lemma~\ref{lm:delta-g} that there exists a special graph in which there exist two clusters of agents $A$ and $B$ with distance $d(A, B) = \min_{u\in A, v\in B}d(u, v) = \Omega\left(\sqrt{(d_{\max}+1)/\lambda_{N-1}(M)}\right)$. Then, we consider an instance where only agents in cluster $A$ receive non-zero losses. 
Based on a reduction argument, the cumulative regrets of agents in cluster $B$ are the same as (up to a constant factor) the cumulative regrets in a single-agent adversarial bandit setting with feedback delay $d(A, B)$ (see Lemma~\ref{lm:delay} in Appendix~\ref{app-lower-full}). Hence, one can show that the cumulative regret of agents in cluster $B$ is $\Omega \left(\sqrt{d(A,B)}\sqrt{T\log K}\right)$.  

We denote with $d(\mathcal{U}, \mathcal{U}^\prime)$ the smallest distance between any two nodes in $\mathcal{U}, \mathcal{U}^\prime \subset \mathcal{V}$, i.e.
$$
d(\mathcal{U}, \mathcal{U}^\prime) = \min_{u\in U, u^\prime\in \mathcal{U}^\prime} d(u, u^\prime)
$$
where $d(u,v)$ is the length of a shortest path connecting $u$ and $u^\prime$.
\begin{definition}[Cluster-based federated algorithms]
    A cluster-based federated algorithm is a multi-agent learning algorithm defined by a partition of graph $\bigcup_{r} \mathcal{U}_r = \mathcal{V}$ 
    % \mv{Why using this fancy notation $\bigcup_{r}$, which is used without being defined? Why not using more commonly used set union notation $\bigcup_r$?} 
    where $\mathcal{U}_r$ is called cluster.
    In the cluster-based federated algorithm, at each round $t$, the action selection probability $p^v_t$ of agent $v\in \mathcal{U}_r$ depends on the history information up to round $t-d(\mathcal{U}_r, \mathcal{U}_{r^\prime})-1$ of all agents $u^\prime\in \mathcal{U}_{r^\prime}$.
\end{definition}

Note that when all agents are in the same cluster $\mathcal{V}$, the centralized federated algorithm in \cite{reda2022near} is an instance of a cluster-based federated algorithm.

\begin{lemma}[Monotonicity]\label{lm:monotone}
    Let $\Pi$ and $\Pi^\prime$ be two sets of all cluster-based federated algorithms with two partitions $\bigcup_{r} \mathcal{U}_r$ and $\bigcup_{s} \mathcal{U}^\prime_s$, respectively. Suppose for any cluster $\mathcal{U}^\prime_{s}$ of $\pi^\prime$, there exists a cluster $\mathcal{U}_r$ of $\pi$ such that $\mathcal{U}^\prime_{s}\subset \mathcal{U}_r$, then 
    $$ \Pi^\prime \subset \Pi \quad \text{and} \quad \min_{\pi\in\Pi} R^v_T(\pi, L ) \leq  \min_{\pi^\prime\in\Pi^\prime} R^v_T(\pi^\prime, L)$$
    for any $L\in [0, 1]^{T\times N\times K}$ and any $v\in \mathcal{V}$.
\end{lemma}
\begin{proof}
    It suffices to show $\Pi^\prime \subset \Pi$.
    Consider a cluster-based federated algorithm $\pi^\prime \in \Pi^\prime$. For any agent $v\in \mathcal{V}$, let $\mathcal{U}^\prime_s$ be the cluster of $v$ in $\Pi^\prime$. By definition of cluster-based procedure, agent $v$'s action selection distribution probability $p^v_t$ depends on the history information up to round $t-d(\mathcal{U}^\prime_s, \mathcal{U}^\prime_{h})-1$ of all agents $u^\prime\in \mathcal{U}^\prime_{h}$.
    
    By the assumption, there exists two subset $\mathcal{U}_{r_1}, \mathcal{U}_{r_2}\subset \mathcal{V}$ such that $\mathcal{U}^\prime_{s} \subset \mathcal{U}_{r_1}$ and $\mathcal{U}^\prime_{h} \subset \mathcal{U}_{r_2}$. Hence $d(\mathcal{U}_{r_1}, \mathcal{U}_{r_2}) \leq d(\mathcal{U}^\prime_{s}, \mathcal{U}^\prime_{h})$, from which it follows $t-d(\mathcal{U}^\prime_{s}, \mathcal{U}^\prime_{h})-1 \leq t - d(\mathcal{U}_{r_1}, \mathcal{U}_{r_2}) - 1$. Hence $\pi^\prime\in \Pi$ which completes the proof.
\end{proof}

% We present a regret lower bound on any distributed online learning procedure with 2 clusters. \mv{This is a strange sentence as what immediately follows is a lemma from literature.}

\begin{lemma}\label{lm:delta-g}
There exists a graph $\mathcal{G}=(\mathcal{V}, \mathcal{E})$ with $N$ nodes and a matrix $M\in \mathcal{M}_{\mathcal{G}}$, together with two subsets of nodes $I_0, I_1 \subset \mathcal{V}$ of size $\left|I_0\right|=\left|I_1\right| \geq N/4$ and such that
$$
d\left(I_0, I_1\right) \geq \Tilde{\Delta},
$$
where $d\left(I_0, I_1\right)$ is the shortest-path distance in $\mathcal{G}$ between the two sets and $$\Tilde{\Delta} =  \frac{\sqrt{2}}{3} \sqrt{\frac{1 + d_{\max}} {\lambda_{N-1}(M)}}.$$
\end{lemma}

\begin{proof}
    From Lemma~24 in \cite{scaman2019optimal}, there exists exists a graph $\mathcal{G}=(\mathcal{V}, \mathcal{E})$ with $N$ nodes and a matrix $M\in \mathcal{M}_{\mathcal{G}}$, together with two subsets of nodes $I_0, I_1 \subset \mathcal{V}$ of size $\left|I_0\right|=\left|I_1\right| \geq N/4$ and such that
$$
d\left(I_0, I_1\right) \geq \frac{\sqrt{2}}{3} \sqrt{\frac{\lambda_1(M)} {\lambda_{N-1}(M)}}.
$$
We show that $\lambda_1(M) \geq 1 + d_{\max}$.
To see this, note that $\lambda_1(M)$ is the Rayleigh quotient $\max _{x \neq 0} \frac{x^T M x}{x^T x}$.
By the definition of the Laplacian matrix,
$$
x^T M x=\sum_{(v, u) \in \mathcal{E}}\left(x_v-x_u\right)^2
$$
Let $v$ be the vertice whose degree is $d_{\max}$ and let
$$
x_u:= \begin{cases} \sqrt{\frac{d_{\max}}{1+d_{\max}}} & \text { if } u=v \\ - \frac{1}{1+d_{\max}} \sqrt{1+\frac{1}{d_{\max}}} & \text { if } u \neq v \text { and } v_i \text { is adjacent to } v_j \\ 0 & \text { otherwise }\end{cases}
$$
then
$$
\sum_{(v, u) \in \mathcal{E}}\left(x_v-x_u\right)^2
= d_{\max} \left( \sqrt{\frac{d_{\max}}{1+d_{\max}}} +  \frac{1}{1+d_{\max}} \sqrt{1+\frac{1}{d_{\max}}}\right)^2
= 1 + d_{\max}.
$$
Hence, $\lambda_1(M) \geq 1+ d_{\max}$
\end{proof}

Let $I_0, I_1$ be two subsets of nodes satisfying
$$d(I_0, I_1)\geq \Tilde{\Delta}\quad \text{and} \quad |I_0| = |I_1| = N/4.$$
The number of rounds needed to communicate between any node in $I_0$ and any node $I_1$ is at least $\Tilde{\Delta}$.
\begin{lemma} \label{lm:delay}
    Let  $v_0\in I_0$ and $v_1\in \mathcal{V}\backslash I_0$. Consider a cluster-based federated algorithm with clusters
    $I_0$ and $V\backslash I_0$. Then, any distributed online learning algorithm $\sigma$ for full information feedback setting has an expected regret 
    $$
    R^{v_1}_T \geq \frac{1-o(1)}{4} \sqrt{\frac{\left(\Tilde{\Delta}+1\right)}{2}T\log K}
    $$
    as $T\to\infty$.
\end{lemma}

\begin{proof}
Consider an online learning with expert advice problem with the action set $\mathcal{A}$ over $B$ rounds \citep{cesa1997use}.
Let $\ell_1^\prime, \dots, \ell_B^\prime$ be an arbitrary sequence of losses and $p^\prime_b$ be the action selection distribution at round $b$.
We show that $\sigma$ can be used to design an algorithm for this online learning with expert advice problem, adapted from \cite{cesa2016delay}.

Consider the loss sequences $\{\ell^v_t\}_{t=1}^T$ for each $v\in \mathcal{V}$ with $T = (\Tilde{\Delta}+1)B$ such that
$$
\ell^v_t =
\begin{cases}
    \ell^\prime_{\lceil t/(\Tilde{\Delta}+1) \rceil}, & v\in I_0 \\
    0 & \text{otherwise.} \\
\end{cases}
$$
Let $p^v_t$ be the action select distribution of agent $v\in \mathcal{V}$ running the algorithm $\sigma$.
Define the algorithm for the online learning with expert advice problem as follows:
$$
p^\prime_b =\frac{1}{\Tilde{\Delta}+1} \sum_{s=1}^{\Tilde{\Delta}+1} p^{v_1}_{(\Tilde{\Delta}+1)(b-1)+s}
$$
where $p^v_t = (1/k,\dots,1/k)$ for all $t\leq 1$ and $v\in \mathcal{V}$.

Note that $p^\prime_b$ is defined by $p^{v_1}_{(\Tilde{\Delta}+1)(b-1)+1}, \dots, p^{v_1}_{(\Tilde{\Delta}+1)b}$.
These are in turn defined by $\ell_1^{v_0},\dots, \ell_{(\Tilde{\Delta}+1)(b-1)}^{v_0}$ by the definition of cluster-based communication protocol. Also note that $\lceil t/(\Tilde{\Delta}+1) \rceil \leq b - 1$ for $t\leq (\Tilde{\Delta}+1)b$, hence $p^\prime_b$ is determined by $\ell_1^\prime, \dots, \ell_{b-1}^\prime$.
Therefore $p^\prime_1,\dots, p^\prime_B$ are generated by a legitimate algorithm for online learning with expert advice problem.

Note that the cumulative regret of agent $v_1$ is 
\begin{eqnarray}
    \label{eq:lower-1}
    \sum_{t=1}^T \langle p^{v_1}_t, \bar{\ell}_t \rangle 
    & = & \frac{1}{N} \sum_{t=1}^T \left[ \sum_{v\in I_0} \langle p^{v_1}_t, \ell^v_t \rangle + \sum_{v\not\in I_0} \langle p^{v_1}_t, \ell^v_t \rangle\right]  \nonumber \\ 
    & = & \frac{1}{4} \sum_{t=1}^T \langle p^{v_1}_t, \ell^\prime_{\lceil t/(\Tilde{\Delta}+1) \rceil} \rangle \nonumber \\
    & = & \frac{1}{4} \sum_{b=1}^B \sum_{s=1}^{\Tilde{\Delta}+1} \langle p^{v_1}_{(\Tilde{\Delta}+1)(b-1)+s}, \ell^\prime_{b} \rangle \nonumber \\
    & = & \frac{\Tilde{\Delta}+1}{4} \sum_{b=1}^B  \langle p^{\prime}_{b}, \ell^\prime_{b} \rangle
\end{eqnarray}
the second equality comes from the definition of $\ell^v_t$ and the fourth equality comes from the definition of $p^{\prime}_b$.

Also note that 
\begin{eqnarray}
\label{eq:lower-2}
\min_{i\in \mathcal{A}} \sum_{t=1}^T \bar{\ell}_t(i)
&=& \frac{1}{4} \min_{i\in \mathcal{A}} \sum_{t=1}^T \ell^\prime_{\lceil t/(\Tilde{\Delta}+1) \rceil} (i) \nonumber \\
&=& \frac{\Tilde{\Delta}+1}{4} \min_{i\in \mathcal{A}} \sum_{b=1}^B \ell^\prime_{b} (i)
\end{eqnarray}
From (\ref{eq:lower-1}) and (\ref{eq:lower-2}), it follows that
$$
\sum_{t=1}^T \langle p^{v_1}_t, \bar{\ell}_t \rangle  - \min_{i\in \mathcal{A}} \sum_{t=1}^T \bar{\ell}_t(i)
= \frac{\Tilde{\Delta}+1}{4}
\left[ \sum_{b=1}^B  \langle p^{\prime}_{b}, \ell^\prime_{b} \rangle - \min_{i\in \mathcal{A}} \sum_{b=1}^B \ell^\prime_{b} (i)\right].
$$
There exists a sequence of losses $\ell_1^\prime, \dots, \ell_B^\prime$ such that for any algorithm for online learning with expert advice problem, the expected regret satisfies \citep[Theorem~3.7]{cesa2006prediction}, 
% \mv{Strange citation --- could you refer to a specific theorem in \cite{cesa1997use}?}
$$
\sum_{b=1}^B  \langle p^{\prime}_{b}, \ell^\prime_{b} \rangle - \min_{i\in \mathcal{A}} \sum_{b=1}^B \ell^\prime_{b} (i)
\geq
(1-o(1))\sqrt{\frac{B}{2} \ln K}
$$
Hence, we have 
$$
\sum_{t=1}^T \langle p^{v_1}_t, \bar{\ell}_t \rangle  - \min_{i\in \mathcal{A}} \sum_{t=1}^T \bar{\ell}_t(i)
\geq
\frac{1-o(1)}{4} \sqrt{(\Tilde{\Delta}+1)\frac{T}{2}  \ln K}.
$$
\end{proof}

\subsection{Proof of Theorem~\ref{thm:lower-bound-bandit}} \label{app:lower-bandit}

% \mv{Same comment here as for the proof in the preceeding section - should have a paragraph discussing main steps of the proof, pointing how known results are leveraged, and higlighting any novel ideas.}
The lower bound contains two parts. The first part is derived by using information-theoretic arguments in \cite{shamir2014fundamental} and it  captures the effect of bandit feedback. The second part is inherited from the full-information feedback lower bound in Theorem~\ref{thm:lower-bound-full} by the fact that the regret of an agent in the bandit feedback setting cannot be smaller than the regret in the full-information setting.

Consider a centralized federated algorithm with all the agents in the same cluster $\mathcal{V}$, denoted as $\Pi^C$.
Note that by Lemma~\ref{lm:monotone}, for a federated bandit algorithm $\Pi^G$,
$$ \Pi^G \subset \Pi^C \quad \text{and} \quad \min_{\pi^\prime\in\Pi^C} R^v_T(\pi^\prime, L ) \leq  \min_{\pi\in\Pi^G} R^v_T(\pi, L)$$
for any $L\in [0, 1]^{T\times N\times K}$ and any $v\in \mathcal{V}$.

% For any federated bandit algorithm $\pi\in \Pi^G$, there exists a centralized federated algorithm $\pi^\prime\in \Pi^C$ such that 
% $$ R^v_T(\pi^\prime, L ) \leq  R^v_T(\pi, L)$$
% for any $L\in [0, 1]^{T\times N\times K}$ and any $v\in \mathcal{V}$. \mv{Should this para be removed?}

For any $\pi^\prime\in \Pi^C$, at each round $t$, every agent $v\in \mathcal{V}$ receives $O(|\mathcal{N}(v)|)$ bits since its neighboring agents can choose at most $|\mathcal{N}(v)|$ distinct actions. By Theorem~4 in \cite{shamir2014fundamental}, there exists some distribution $\mathcal{D}$ over $[0,1]^K$ such that loss vectors $\Bar{\ell}_t \overset{i.i.d}{\sim} \mathcal{D}$ for all $t=1,2,\dots, T$  and $\min_{i\in \mathcal{A}} \mathbb{E}\left[ \sum_{t=1}^T \Bar{\ell}_t(a_t(v)) - \sum_{t=1}^T\Bar{\ell}_t(i)\right] = \Omega\left(\min\{T, \sqrt{KT/(1+|\mathcal{N}(v)|)}\}\right)$.

Hence, it follows that
\begin{eqnarray}
    \min_{L} R^v_T(\pi, L)
    &\geq& \min_{L} R^v_T(\pi^\prime, L ) \nonumber \\ 
    & \geq &
    \mathbb{E}_{\Bar{\ell}_t\sim \mathcal{D}}\left[\sum_{t=1}^T\Bar{\ell}_t(a_t(v)) - \min_{i\in \mathcal{A}} \sum_{t=1}^T \Bar{\ell}_t(i)\right] \nonumber \\
    &\geq& 
    \max_{i\in \mathcal{A}} \mathbb{E}_{\Bar{\ell}_t\sim \mathcal{D}}\left[\sum_{t=1}^T\Bar{\ell}_t(a_t(v)) -  \sum_{t=1}^T \Bar{\ell}_t(i)\right] \nonumber \\
    &\geq& \min_{i\in \mathcal{A}}
    \mathbb{E}_{\Bar{\ell}_t\sim \mathcal{D}}\left[\sum_{t=1}^T\Bar{\ell}_t(a_t(v)) - \sum_{t=1}^T \Bar{\ell}_t(i)\right] \nonumber \\ 
    &=& \Omega\left(\min\{T, \sqrt{KT/(1+|\mathcal{N}(v)|)}\}\right) \nonumber 
\end{eqnarray}
where the third inequality comes from Jensen's inequality.

Also note that any federated bandit algorithm for bandit feedback setting is also a federated bandit algorithm for full-information setting, from which it follows 
\begin{eqnarray}
    \min_{L} R^v_T(\pi, L) 
    &\geq&   
     \max\left\{ \Omega\left(\min\{T, \sqrt{KT/(1+|\mathcal{N}(v)|)}\}\right), \Omega \left(\sqrt[4]{\frac{1+d_{\max}}{\lambda_{N-1}(M)} } \sqrt{ T\log K} \right) \right\} \nonumber \\
    &=&
    \Omega\left( \min\left\{T, \max\left\{\sqrt{K/(1+|\mathcal{N}(v)|)}, \sqrt[4]{\frac{1+d_{\max}}{\lambda_{N-1}(M)} } \sqrt{\log K} \right\}\sqrt{T}\right\} \right). \nonumber 
\end{eqnarray}

\subsection{Auxiliary lemmas}
Here we present some auxiliary lemmas which are used in the proof of Theorem~\ref{thm:upper-bound:static}.
% \mv{Explain for what are lemmas presented in this section used for.}
Recall that $\hat{\ell}_t$ and $\bar{z}_t$ are the average instant loss estimator and average cumulative loss, 
\begin{equation*}
    f_t = \frac{1}{N}\sum_{v\in \mathcal{V}} g^v_t \quad \text{and} \quad \bar{z}_t = \frac{1}{N}\sum_{v\in \mathcal{V}} z^v_t
\end{equation*}
and $y_t$ is action distribution to minimize the regularized average cumulative loss
$$
y_t(i) = \frac{\exp\{-\eta_{t-1} \bar{z}_t(i)\}}{\sum_{j\in A} \exp\{-\eta_{t-1} \bar{z}_t(j)\}}
$$
\begin{lemma}
\label{lm:z-bar}
For each time step $t=1,\dots,T$,
\begin{equation*}
    \bar{z}_{t+1} = \bar{z}_t +  f_t
\end{equation*}
and
$$\max\{ \|g_t^v\|_{\ast}, \|f_t\|_{\ast}\} \leq \frac{K}{\gamma_t}
$$
\end{lemma}
\begin{proof}
\begin{eqnarray}
\bar{z}_{t+1} 
&=& \frac{1}{N}\sum_{v\in \mathcal{V}} z^v_{t+1}\nonumber \\
&=& \frac{1}{N}\sum_{v\in \mathcal{V}} \sum_{u: (u,v)\in \mathcal{E}} W_{u,v} z_t^u + \frac{1}{N}\sum_{v\in \mathcal{V}}g_t^v\nonumber \\
&=& \frac{1}{N}\sum_{v\in \mathcal{V}} z^v_{t} + \frac{1}{N}\sum_{v\in \mathcal{V}}g_t^v\nonumber \\
&=& \bar{z}_t + f_t \nonumber
\end{eqnarray}
where the second equality comes from Line~7 in Algorithm \ref{algo:FedExp3} and the third equality comes from the double-stochasticity of $W$.

Noting that $p^v_t(i) \geq \gamma / K$ for all $v\in \mathcal{V}$, $i\in \mathcal{A}$ and $t\in \{1,\ldots, T\}$, it follows that
$$
\|g^v_t\|_\ast = \frac{\ell^v_t(a^v_t)}{p^v_t(a^v_t)} \leq \frac{K}{\gamma_t}
\quad
\text{and}
\quad
\|f_t\|_\ast \leq \frac{1}{N}\sum_{v\in \mathcal{V}}  \| g^v_t\|_\ast \leq \frac{K}{\gamma_t}.
$$
\end{proof}

\begin{lemma}\label{lm:loss-estimator}
For any $v\in \mathcal{V}$ and $t\geq 1$, it holds that
$$
\mathbb{E}\left[g_t^v \mid \mathcal{F}_{t-1}\right] = \ell^v_t \text { and } \mathbb{E}\left[f_t \mid \mathcal{F}_{t-1}\right] = \Bar{\ell}_t.
$$
with
$$
\mathbb{E}\left[\|f_t\|_\ast \right] \leq K \text { and } \mathbb{E}\left[\|f_t\|_\ast^2 \right] \leq \frac{K^2}{\gamma_t}
$$
\end{lemma}
\begin{proof}
Note that $p_t^v$ is determined by $\mathcal{F}_{t-1}$, hence
$$
\mathbb{E}\left[g_t^v(i) \mid \mathcal{F}_{t-1}\right] = \frac{\ell^v_t(i) }{p^v_t(i) }
\mathbb{E}\left[\mathbb{I}\left\{a^v_t = i \right\} \mid \mathcal{F}_{t-1}\right]
=
\frac{\ell^v_t(i) }{p^v_t(i) } p^v_t(i)
= \ell^v_t(i) 
$$
and
$$
\mathbb{E}\left[\|g_t^v\|_\ast \right]
= \mathbb{E}\left[\frac{\ell_t^v(a_t^v)}{p_t^v(a^v_t)} \right]
= \mathbb{E}\left[ \mathbb{E}\left[\frac{\ell_t^v(a_t^v)}{p_t^v(a^v_t)}  \mid \mathcal{F}_{t-1} \right] \right]
= \mathbb{E}\left[ \sum_{i\in \mathcal{A}} p^v_t(i) \frac{\ell_t^v(i)}{p_t^v(i)} \right]
= \sum_{i\in \mathcal{A}} \ell_t^v(i) \leq K
$$
where the last inequality comes from $\ell_t^v(i) \leq 1$.
Since $f_t (i) = \frac{1}{N} \sum_{v\in \mathcal{V}} g^v_t$,
it follows that
$$
\mathbb{E}\left[f_t (i) \mid \mathcal{F}_{t-1}\right]
= \frac{1}{N}\sum_{v\in \mathcal{V}} \ell^v_t(i) = \Bar{\ell}_t(i).
$$
and
$$
\mathbb{E}\left[\|f_t\|_\ast \right] \leq \frac{1}{N}\sum_{v\in \mathcal{V}} \mathbb{E}\left[\|g_t^v\|_\ast \right] \leq K
$$
which comes from Jensen's inequality.
Notice that
$$
\mathbb{E}\left[\|g_t^v\|_\ast^2 \right]
= \mathbb{E}\left[\frac{\ell_t^v(a_t^v)^2}{p_t^v(a^v_t)^2} \right]
= \mathbb{E}\left[ \mathbb{E}\left[\frac{\ell_t^v(a_t^v)^2}{p_t^v(a^v_t)^2}  \mid \mathcal{F}_{t-1} \right] \right]
= \mathbb{E}\left[ \sum_{i\in \mathcal{A}} p^v_t(i) \frac{\ell_t^v(i)^2}{p_t^v(i)^2} \right]
= \mathbb{E}\left[ \sum_{i\in \mathcal{A}} \frac{1}{p_t^v(i)} \right] \leq \frac{K^2}{\gamma_t}
$$
the last inequality comes from $p_t^v(i) \geq \gamma_t / K$.
Again, from Jensen's inequality, it follows
$$
\mathbb{E}\left[\|f_t\|_\ast^2 \right] \leq \frac{1}{N}\sum_{v\in \mathcal{V}} \mathbb{E}\left[\|g_t^v\|_\ast^2 \right] \leq \frac{K^2}{\gamma_t}.
$$
\end{proof}

Before presenting the next lemma, we recall the definition of strongly-convex functions and Fenchel duality. A function $\phi$ is said to be $\alpha$-strongly convex function on a convex set $\mathcal{X}$ if 
$$
\phi(x^\prime) \geq \phi(x) + \langle \nabla\phi(x), x^\prime- x \rangle + \frac{1}{2} \alpha\|x^\prime- x\|^2
$$
for all $x^\prime, x\in \mathcal{X}$, for some $\alpha \geq 0$.

Let $\phi^\ast$ denote the \emph{Fenchel conjugate} of $\phi$, i.e.,
$$
\phi^{\ast}(y)=\max_{x \in \mathcal{X}}\{\langle x, y\rangle-\phi(x)\}
$$
with the projection, 
$$
\nabla \phi^{\ast}(y)=\arg\max_{x \in \mathcal{X}} \{\langle x, y\rangle-\phi(x)\}.
$$

\begin{lemma}
\label{lm:lipschitz}
Let $\psi$ the normalized negative entropy function \cite{lattimore_szepesvari_2020} on $\mathcal{P}_{K-1} = \{x\in [0,1]^K: \sum_{i=1}^K x(i) = 1\}$ 
% \mv{It may be better to provide this definition before the lemma.}
,
$$
\psi_\eta (x) = \frac{1}{\eta} \sum_{i=1}^k x(i) \left(\log(x(i)) - 1\right).
$$
For all $t=1,\dots, T$,
it holds that
$$
x^{v}_{t} = \underset{x\in \mathcal{P}_{K-1}}{\operatorname{argmin}}\{\langle x, z^{v}_t \rangle + \psi_{\eta_{t}}(x)\} = \nabla \psi_{\eta_{t-1}}^\ast(-z^{v}_{t})
$$
with $\mathcal{X} = \mathcal{P}_{K-1}$ and 
$$
y_t = \underset{x\in \mathcal{P}_{K-1}}{\operatorname{argmin}}\{\langle x, \bar{z}_t \rangle + \psi_{\eta_{t-1}}(x)\}  = \nabla \psi_{\eta_{t-1}}^\ast(-\bar{z}_t)
$$
Furthermore, it holds
$$
\|x^{v}_t - y_t\| \leq \eta_{t-1} \|z^{v}_t - \bar{z}_t\|_\ast.
$$
\end{lemma}

\begin{proof}
We prove for $y_t = \underset{x\in \mathcal{P}_{K-1}}{\operatorname{argmin}}\{\langle x, \bar{z}_t \rangle + \psi_{\eta_{t-1}}(x)\}$ whose argument also applies to $x^v_t$.

Notice it suffices to consider the minimization problem
\begin{eqnarray}
    \min_{x\in  \mathcal{P}_{K-1}} & \eta_{t-1} \sum_{k=1}^K x(i) \Bar{z}_t(i) +  \sum_{i=1}^k x(i) \log(x(i)) \nonumber \\
    \text{subject to} & \sum_{k=1}^K x(i) = 1. \nonumber
\end{eqnarray}
It suffices to consider the Lagrangian, 
$$
\mathcal{L} = - \eta_{t-1} \sum_{k=1}^K x(i) \Bar{z}_t(i) -  \sum_{i=1}^k x(i) \log(x(i))  - \lambda \left(\sum_{k=1}^K x(i) - 1\right). 
$$
Consider the first-order conditions for all $i=1,\dots, K$
\begin{eqnarray}
    \frac{\partial \mathcal{L}}{\partial x(i)} &=& -\eta_{t-1} \Bar{z}_t(i) - \log(x(i)) - 1 - \lambda = 0 \nonumber
\end{eqnarray}
which gives $x(i) = \exp\{-\eta_{t-1}\Bar{z}_t(i)\} / \exp\{1+\lambda\}$ for all $i=1,\dots, K$.
Plugging into the constraint $\sum_{k=1}^K x(i) = 1$ together with the definition of Fenchel duality \cite{hiriart-urruty_convex_2010} completes the proof for $y_t$.

Note that the normalized negative entropy $\psi(x)$ is $1$-strongly convex,
$$
\psi(x^\prime) \geq \psi(x) + \langle \nabla\psi(x), x^\prime- x \rangle + \frac{1}{2}\|x^\prime- x\|^2
$$
Multiplying $1/\eta_{t-1}$ both sides of the inequality yields that $\psi_{\eta_{t-1}}(x)$ is $1/\eta_{t-1}$-strongly convex. 
% \mv{The concept of "modulus" does not appear to be defined earlier. Should define somewhere $\alpha$-strong convexity for parameter $\alpha$.}
By Theorem~4.2.1 in \cite{hiriart-urruty_convex_2010}, we have that $\nabla \psi^\ast_{\eta_{t-1}}(z)$ is $\eta_{t-1}$-Lipschitz. 

It follows that
$$
\|p^{v}_t - \bar{p}_t\| = \|\nabla \psi_{\eta}^\ast(-z^{v}_t) - \nabla \psi_{\eta}^\ast(-\bar{z}_t)\| \leq \eta_{t-1} \|\bar{z}_t - z^{v}_t\|_\ast .
$$
\end{proof}

We state an upper bound on the network disagreement on the cumulative loss estimators from \cite{hosseini2013online}.
\begin{lemma}
\label{lm:gossip}
For any $v\in \mathcal{V}$ and $t=1,2,\dots,T$, $$
\|\bar{z}_t - z^{v}_t\|_\ast
\leq \frac{K}{\gamma_T}\left( \frac{\min\{2\log T + \log n, \sqrt{n}\}}{1-\sigma_2(W)} +3\right) = \frac{K}{\gamma_T} C_W
$$
where $\sigma_2(W)$ is the second largest singular value of $W$.
\end{lemma}
\begin{proof}
    From Lemma~\ref{lm:z-bar}, it follows that $\|g^v_t\|_\ast \leq K/\gamma_t$. Since $\{\gamma_{t}\}$ is non-increasing, let $L = K/\gamma_T$ in Eq.~(29) in \cite{duchi2011dual} and Lemma~6 in \cite{hosseini2013online} completes the proof.
\end{proof}

\subsection{Proof of Theorem~\ref{thm:upper-bound:static}} \label{proof:upper-bound}

Let $i^\ast = \arg\min_{i\in \mathcal{A}} \sum_{t=1}^T \Bar{\ell}_t(i)$.
Note that $p^v_t$ is determined by $\mathcal{F}_{t-1}$ and $\mathbb{E}\left[f_t \mid \mathcal{F}_{t-1}\right] = \Bar{\ell}_t$ from Lemma~\ref{lm:loss-estimator}. It follows that for each agent $v\in \mathcal{V}$
\begin{eqnarray} \label{eq:jensen}
    R^v_T
&=& \mathbb{E}\left[\sum_{t=1}^T \langle \Bar{\ell}_t, p^v_t \rangle -  \sum_{t=1}^T \Bar{\ell}_t(i^\ast ) \right] \nonumber \\
&=&  \mathbb{E}\left[
\sum_{t=1}^T \langle \mathbb{E}\left[ f_t \mid \mathcal{F}_{t-1} \right], p^v_t \rangle - \sum_{t=1}^T  \mathbb{E}\left[ f_t (i^\ast) \mid \mathcal{F}_{t-1} \right] \right]\nonumber \\
&= & 
\mathbb{E}\left[\sum_{t=1}^T \langle  f_t  , p^v_t \rangle  -  \sum_{t=1}^T   f_t (i^\ast) \right] \nonumber \\
\end{eqnarray}
By the definition of $p^v_t$, it follows
\begin{eqnarray}
    R^v_T
    &=& \mathbb{E}\left[\sum_{t=1}^T \left( \langle f_t, (1-\gamma) x^v_t + \gamma x^v_1 \rangle - f_t (i^\ast)  \right) \right] \nonumber \\
    &=& \mathbb{E}\left[\sum_{t=1}^T (1-\gamma_t) \left(\langle f_t, x^v_t \rangle - f_t (i^\ast) \right)\right] + \sum_{t=1}^T \gamma_t  \mathbb{E}\left[ \left( \langle f_t, x^v_1 \rangle - f_t (i^\ast) \right) \right] \nonumber \\
    &=&  \mathbb{E}\left[\sum_{t=1}^T (1-\gamma_t) \left(\langle f_t, x^v_t \rangle - f_t (i^\ast) \right)\right] + \sum_{t=1}^T \gamma_t  \left( \langle \Bar{\ell}_t, x^v_1 \rangle - \Bar{\ell}_t (i^\ast) \right) \nonumber \\
    &\leq& \mathbb{E}\left[\sum_{t=1}^T \left(\langle f_t, x^v_t \rangle - f_t (i^\ast) \right)\right] + \sum_{t=1}^T \gamma_t \nonumber \\
    &=& 
    \underbrace{\mathbb{E}\left[\sum_{t=1}^T \left(\langle f_t, y^v_t \rangle - f_t (i^\ast) \right)\right]}_{(\rom{1})} + 
    \underbrace{\mathbb{E}\left[\sum_{t=1}^T \langle f_t, x^v_t - y_t^v \rangle\right]}_{(\rom{2})} + \sum_{t=1}^T \gamma_t \nonumber
\end{eqnarray}
where the first inequality comes from the fact that $\gamma_t>0$ and the fact that $\|\Bar{\ell}_t\|_\ast \leq \sum_{v\in \mathcal{V}} 1/N \|\Bar{\ell}_t^v\|_\ast \leq 1$.

From Lemma~\ref{lm:z-bar}, it follows
$\Bar{z}_t = \sum_{s=1}^{t-1} f_s$.
Hence, it follows from Lemma~\ref{lm:lipschitz} that
$$
y_t = \underset{x \in \mathcal{P}_{K-1}}{\arg \min }\left\{\sum_{s=1}^t\langle f_s, x\rangle+\frac{1}{\eta_{t-1}} \psi(x)\right\}
$$
From Lemma 3 in \cite{duchi2011dual} and  Corollary~28.8 in \cite{lattimore_szepesvari_2020}, we have
\begin{equation}
    \label{eq:rom-1}
    (\rom{1}) \leq \frac{1}{2} \sum_{t=1}^T \eta_{t-1} \mathbb{E}\left[ \|f_t\|_*^2 \right]+\frac{1}{\eta_{T}} \log(K)
\end{equation}
because $\{\eta_t\}$ is a non-increasing sequence.
Note that 
$$
\|x^{v}_t - y_t\| \leq \eta_{t-1} \|z^{v}_t - \bar{z}_t\|_\ast
$$
by Lemma~\ref{lm:lipschitz}. This yields that
\begin{equation}
\label{eq:rom-2}
    (\rom{2}) \leq \sum_{t=1}^T \eta_{t-1} \mathbb{E}\left[ \|f_t\|_\ast   \|z^{v}_t - \bar{z}_t\|_\ast \right]
\end{equation}
Plugging Equations~(\ref{eq:rom-1}) and (\ref{eq:rom-2}) into (\rom{1}) yields that
\begin{eqnarray}
    \label{eq:before-tuning}
    R^v_T  
    &\leq&
    \frac{1}{2} \sum_{t=1}^T \eta_{t-1} \mathbb{E}\left[ \|f_t\|_*^2 \right]+\frac{1}{\eta_{T}} \log(K) + \sum_{t=1}^T \eta_{t-1} \mathbb{E}\left[ \|f_t\|_\ast   \|z^{v}_t - \bar{z}_t\|_\ast \right] + \sum_{t=1}^T \gamma_t  \nonumber \\
    &\leq& \frac{1}{2} \sum_{t=1}^T \eta_{t-1} \mathbb{E}\left[ \|f_t\|_*^2 \right] + \frac{K}{\gamma_T} C_W \sum_{t=1}^T \eta_{t-1} \mathbb{E}\left[ \|f_t\|_\ast \right] + \sum_{t=1}^T \gamma_t +\frac{1}{\eta_{T}} \log(K) \nonumber \\
    &\leq&
    \frac{K^2}{2} \sum_{t=1}^T \frac{\eta_{t-1}}{\gamma_t}+ \frac{K^2}{\gamma_T} C_W \sum_{t=1}^T \eta_{t-1}  + \sum_{t=1}^T \gamma_t +\frac{1}{\eta_{T}} \log(K) \nonumber 
\end{eqnarray}
where the second inequality comes from Lemma~\ref{lm:gossip} and the third inequality comes from Lemma~\ref{lm:loss-estimator}.

Let 
$$
\gamma_t = \sqrt[3]{\frac{\left(C_W + \frac{1}{2}\right)K^2\log K}{t}}
\quad
\text{and}
\quad
\eta_t = \frac{\log K}{T \gamma_T} = \sqrt[3]{\frac{(\log K)^2}{ \left(C_W + \frac{1}{2}\right) K^2 T^2}}.
$$
Then, for every $v\in \mathcal{V}$, we have
\begin{eqnarray}
    R^v_T &\leq& \frac{3}{8} \sqrt[3]{\frac{K^2\log K}{\left(C_W + \frac{1}{2}\right)^2}} T^{\frac{2}{3}} + \sqrt[3]{K^2\log K \frac{C_W^3}{\left(C_W + \frac{1}{2}\right) ^2}} T^{\frac{2}{3}} + \frac{3}{2} \sqrt[3]{\left(C_W + \frac{1}{2}\right)K^2\log K} T^{\frac{2}{3}} + \sqrt[3]{\left(C_W + \frac{1}{2}\right)K^2\log K }  T^{\frac{2}{3}} \nonumber \\
    &\leq & \frac{3}{4}  \sqrt[3]{K^2\log K} T^{\frac{2}{3}} + \sqrt[3]{C_WK^2\log K} T^{\frac{2}{3}} + \frac{5\sqrt[3]{2}}{2} \sqrt[3]{C_WK^2\log K} T^{\frac{2}{3}}  \nonumber \\
    &\leq& 5\sqrt[3]{C_W K^2\log K} T^{\frac{2}{3}}. \nonumber
\end{eqnarray}

\end{document}